\pdfoutput=1
\documentclass[twoside,11pt]{article}

\usepackage[utf8]{inputenc}
\usepackage{amsmath, amssymb, amsthm}
\usepackage{graphicx}
\usepackage{hyperref}
\usepackage{geometry}
\usepackage{algorithm}
\usepackage{algorithmic}
\usepackage{natbib}
\usepackage{booktabs}
\usepackage{subcaption}
\usepackage{epigraph}
\newtheorem{theorem}{Theorem}[section]

\newtheorem{proposition}[theorem]{Proposition}
\newtheorem{corollary}[theorem]{Corollary}
\theoremstyle{definition}
\newtheorem{definition}[theorem]{Definition}
\newtheorem{assumption}[theorem]{Assumption}

\theoremstyle{remark}
\newtheorem{remark}[theorem]{Remark}

\newcommand{\E}{\mathbb{E}}
\newcommand{\R}{\mathbb{R}}
\newcommand{\cF}{\mathcal{F}}
\newcommand{\cX}{\mathcal{X}}
\newcommand{\cA}{\mathcal{A}}
\newcommand{\Lsur}{\tilde{L}}
\newcommand{\thhat}{\hat{\theta}}
\newcommand{\thtil}{\tilde{\theta}}
\newcommand{\Deltan}{\Delta_n}
\newcommand{\op}{o_p}
\newcommand{\Op}{O_p}

\title{Likelihood-Preserving Embeddings for Statistical Inference}

\author{
  Deniz Akdemir \\
  \texttt{deniz.akdemir.work@gmail.com}
}

\date{\today}

\begin{document}

\maketitle

\vspace{1em}
\begin{quote}
\textit{``The object of statistical methods is the reduction of data.''}
\hfill --- R.A.~Fisher \citep{fisher1925}

\vspace{0.5em}
\textit{``The statistic chosen should summarise the whole of the relevant information supplied by the sample. This may be called the Criterion of Sufficiency.''}
\hfill --- R.A.~Fisher \citep{fisher1922}

\vspace{0.5em}
\textit{``...In that Empire, the Art of Cartography attained such Perfection that the map of a single Province occupied the entirety of a City...''}
\hfill --- Jorge Luis Borges, ``On Exactitude in Science'' \citep{borges1946}
\end{quote}
\vspace{1em}

\begin{abstract}
Modern machine learning embeddings provide powerful compression of high-dimensional data, yet they typically destroy the geometric structure required for classical likelihood-based statistical inference. This paper develops a rigorous theory of \textit{likelihood-preserving embeddings}: learned representations that can replace raw data in likelihood-based workflows---hypothesis testing, confidence interval construction, model selection---without altering inferential conclusions.

We introduce the \textit{Likelihood-Ratio Distortion} metric $\Deltan$, which measures the maximum error in log-likelihood ratios induced by an embedding. Our main theoretical contribution is the \textit{Hinge Theorem}, which establishes that controlling $\Deltan$ is necessary and sufficient for preserving inference. Specifically, if the distortion satisfies $\Deltan = \op(1)$, then (i) all likelihood-ratio based tests and Bayes factors are asymptotically preserved, and (ii) surrogate maximum likelihood estimators are asymptotically equivalent to full-data MLEs. We prove an impossibility result showing that universal likelihood preservation requires essentially invertible embeddings, motivating the need for model-class-specific guarantees. We then provide a constructive framework using neural networks as approximate sufficient statistics, deriving explicit bounds connecting training loss to inferential guarantees. Experiments on Gaussian and Cauchy distributions validate the sharp phase transition predicted by exponential family theory, and applications to distributed clinical inference demonstrate practical utility.
\end{abstract}

\section{Introduction}

The last decade has witnessed remarkable success in learning low-dimensional representations of complex data. Neural network embeddings now routinely compress images, text, and molecular structures into fixed-length vectors that support downstream prediction tasks. Yet a fundamental tension exists between these learned representations and classical statistical inference. The central challenge we address is this: \textit{can neural network embeddings be made compatible with classical statistical inference, and if so, under what conditions?}

\subsection{The Core Problem}

Consider a scientist who observes data $X_1, \ldots, X_n$ and wishes to test a hypothesis, construct a confidence interval, or compare models. Classical statistical methods rely on the \textit{likelihood function} $L_n(\theta) = \prod_{i=1}^n p(X_i | \theta)$, whose geometry---its curvature, its maxima, its ratios at different parameter values---encodes all information relevant for inference. When raw data $X$ is replaced by an embedding $Z = T_\phi(X)$, this geometric structure is typically distorted or destroyed, rendering standard inferential guarantees invalid.

This problem is not merely academic. The framework enables valid inference in scenarios constrained by privacy, bandwidth, or computational cost:
\begin{itemize}
    \item \textbf{Distributed Healthcare:} Hospitals conduct joint likelihood-based clinical trials without sharing patient-level data, circumventing regulations like HIPAA while enabling exact frequentist inference.
    \item \textbf{High-Energy Physics:} Trigger systems deploy encoders on detector hardware to compress petabytes of raw event data into embeddings in real-time, preserving the ability to perform hypothesis tests about new particles.
    \item \textbf{Genomics:} Researchers compress high-dimensional genetic data into low-dimensional summaries that preserve power for association testing.
    \item \textbf{Finance:} Banks compute embeddings of private ledgers; regulators aggregate them to assess systemic risk without exposing proprietary trading positions.
    \item \textbf{Large Language Models:} Massive contexts (entire books, conversation histories) can be compressed into ``validity vectors'' that support provable auditing of model knowledge or copyright compliance.
\end{itemize}

\subsection{Our Contribution}

This paper develops a foundational theory of \textit{likelihood-preserving embeddings}: representations that preserve the validity of likelihood-based inference. Our contributions are:

\begin{enumerate}
    \item \textbf{Likelihood-Ratio Distortion.} We introduce $\Deltan$, a metric quantifying the maximum error in log-likelihood ratios induced by an embedding (Definition~\ref{def:distortion}). This is the correct quantity to control because inference depends on likelihood \textit{ratios}, not absolute values.
    
    \item \textbf{The Hinge Theorem.} We prove (Theorem~\ref{thm:hinge}) that controlling $\Deltan$ is necessary and sufficient for preserving frequentist and Bayesian inference: tests, confidence intervals, Bayes factors, and MLEs.
    
    \item \textbf{Impossibility Results.} We show (Theorem~\ref{thm:nfl}) that universal likelihood preservation across all models requires essentially invertible embeddings, clarifying why model-class-specific guarantees are necessary.
    
    \item \textbf{Constructive Framework.} We develop a neural network training procedure that directly targets $\Deltan$, with explicit bounds connecting training loss to inferential guarantees (Section~\ref{sec:construction}).
    
    \item \textbf{Experimental Validation.} We validate the theory on canonical examples (Gaussian, Cauchy) demonstrating the predicted sharp/smooth phase transitions, and on a practical application in distributed inference.
\end{enumerate}

\subsection{Paper Organization}

Section~\ref{sec:setup} formalizes the problem and states regularity conditions. Section~\ref{sec:theory} presents the Hinge Theorem and impossibility results. Section~\ref{sec:construction} develops the constructive neural network framework. Section~\ref{sec:related} positions our work relative to simulation-based inference, information bottleneck, and federated learning. Section~\ref{sec:experiments} presents experimental validation. Section~\ref{sec:discussion} discusses limitations and future directions.

%==============================================================================
\section{Problem Setup and Definitions}
\label{sec:setup}
%==============================================================================

\subsection{Statistical Model}

Let $(\cX, \cA)$ be a measurable space (the sample space) and let $\Theta \subset \R^p$ be the parameter space. We consider a parametric family of probability measures:
\begin{equation}
    \cF = \{P_\theta : \theta \in \Theta\}
\end{equation}
dominated by a $\sigma$-finite measure $\mu$ on $(\cX, \cA)$. For each $\theta \in \Theta$, write $p(\cdot|\theta) = dP_\theta/d\mu$ for the density.

Given i.i.d.\ observations $X_1, \ldots, X_n \sim P_{\theta_0}$ for some unknown $\theta_0 \in \Theta$, the \textbf{log-likelihood function} is:
\begin{equation}
    L_n(\theta) = \sum_{i=1}^n \ell_\theta(X_i), \quad \text{where } \ell_\theta(x) = \log p(x|\theta).
\end{equation}

\subsection{Regularity Conditions}

We impose standard regularity conditions to ensure well-behaved likelihood-based inference.

\begin{assumption}[Identifiability]
\label{ass:ident}
The map $\theta \mapsto P_\theta$ is injective: if $P_\theta = P_{\theta'}$ then $\theta = \theta'$.
\end{assumption}

\begin{assumption}[Smoothness]
\label{ass:smooth}
For $P_{\theta_0}$-almost all $x$, the map $\theta \mapsto \ell_\theta(x)$ is twice continuously differentiable. The derivatives can be passed under the integral sign:
\begin{equation}
    \nabla_\theta \int p(x|\theta) d\mu(x) = \int \nabla_\theta p(x|\theta) d\mu(x).
\end{equation}
\end{assumption}

\begin{assumption}[Fisher Information]
\label{ass:fisher}
The Fisher information matrix
\begin{equation}
    I(\theta) = \E_\theta\left[ \nabla \ell_\theta(X) \nabla \ell_\theta(X)^\top \right] = -\E_\theta\left[ \nabla^2 \ell_\theta(X) \right]
\end{equation}
exists and is positive definite for all $\theta \in \Theta$.
\end{assumption}

\begin{assumption}[Uniform Integrability]
\label{ass:unif}
There exists a function $M: \cX \to [0, \infty)$ with $\E_{\theta_0}[M(X)] < \infty$ such that for all $\theta$ in a neighborhood of $\theta_0$:
\begin{equation}
    \left| \ell_\theta(x) \right| \leq M(x), \quad 
    \left\| \nabla \ell_\theta(x) \right\| \leq M(x), \quad 
    \left\| \nabla^2 \ell_\theta(x) \right\| \leq M(x).
\end{equation}
\end{assumption}

Under these assumptions, the maximum likelihood estimator $\thhat = \arg\max_\theta L_n(\theta)$ is consistent and asymptotically normal:
\begin{equation}
    \sqrt{n}(\thhat - \theta_0) \xrightarrow{d} N(0, I(\theta_0)^{-1}).
\end{equation}

\subsection{Likelihood-Preserving Embeddings}

We now define the central objects of study.

\begin{definition}[Embedding]
\label{def:embedding}
A \textbf{learned embedding} is a measurable function $T_\phi: \cX \to \R^m$, where $m$ is the embedding dimension and $\phi$ indexes the parameters of the embedding (e.g., neural network weights).
\end{definition}

To preserve the exchangeability and additivity inherent in i.i.d.\ likelihoods, we aggregate individual embeddings via summation:

\begin{definition}[Dataset Embedding]
\label{def:dataset}
The \textbf{dataset embedding} is the empirical average:
\begin{equation}
    S_\phi(X_{1:n}) = \frac{1}{n} \sum_{i=1}^n T_\phi(X_i) \in \R^m.
\end{equation}
\end{definition}

\begin{remark}
The choice of averaging (rather than, say, concatenation or attention-based aggregation) is deliberate: it ensures that the dataset embedding is a \textit{U-statistic} and inherits desirable statistical properties. Alternative aggregation schemes (max-pooling, attention, deep sets) are possible and discussed in Section~\ref{sec:discussion}.
\end{remark}

\begin{definition}[Decoder]
\label{def:decoder}
A \textbf{decoder} is a function $h_\psi: \Theta \times \R^m \to \R$ that maps a parameter value and a dataset embedding to an approximate per-sample log-likelihood contribution (a real number, possibly negative).
\end{definition}

\begin{definition}[Surrogate Likelihood]
\label{def:surrogate}
The \textbf{surrogate log-likelihood} is:
\begin{equation}
    \Lsur_n(\theta) = n \cdot h_\psi(\theta, S_\phi(X_{1:n})).
\end{equation}
The factor $n$ ensures that the surrogate scales like the true log-likelihood.
\end{definition}

The central question is: under what conditions does $\Lsur_n$ preserve the inferential content of $L_n$?

\subsection{Pointwise Approximation Error}

Fisher's approach to sufficiency suggests we should ask: can the embedding reconstruct the log-likelihood \textit{pointwise} for each parameter value? This leads to our primary definition:

\begin{definition}[Pointwise Approximation Error]
\label{def:pointwise}
The \textbf{pointwise approximation error} of an embedding $(T_\phi, h_\psi)$ is:
\begin{equation}
    \varepsilon_n(\phi, \psi) = \sup_{\theta \in \Theta} \left| \frac{1}{n}L_n(\theta) - h_\psi(\theta, S_\phi(X_{1:n})) \right|.
\end{equation}
We normalize by $n$ so the target converges to a finite limit.
\end{definition}

\begin{definition}[$\varepsilon$-Sufficient Embedding]
\label{def:sufficient}
An embedding $(T_\phi, h_\psi)$ is \textbf{$\varepsilon$-sufficient} for model class $\cF$ if $\varepsilon_n = \op(1/n)$ as $n \to \infty$.
\end{definition}

\begin{remark}[Why $\op(1/n)$?]
The rate $\varepsilon_n = \op(1/n)$ is natural: it ensures preservation of absolute likelihood values at the $O(1)$ scale required for information criteria (AIC/BIC) and posterior normalization. Slower rates (e.g., $\varepsilon_n = \op(1)$) suffice for likelihood-ratio preservation but may fail to preserve model selection.
\end{remark}

This definition directly generalizes Fisher's factorization criterion: if $\varepsilon_n = 0$, we recover exact sufficiency.

\subsection{Pointwise Implies Ratio Preservation}

Why define sufficiency via pointwise approximation rather than likelihood ratios? Because pointwise convergence is \textit{stronger}---it implies ratio preservation automatically:

\begin{proposition}[Pointwise Implies Ratio Preservation]
\label{prop:pointwise_to_ratio}
If $\varepsilon_n \leq \varepsilon$, then the \textbf{Likelihood-Ratio Distortion}
\begin{equation}
\label{def:distortion}
    \Deltan = \sup_{\theta,\theta' \in \Theta} \left| (L_n(\theta) - L_n(\theta')) - (\Lsur_n(\theta) - \Lsur_n(\theta')) \right|
\end{equation}
satisfies $\Deltan \leq 2n\varepsilon_n$.
\end{proposition}

\begin{proof}
\begin{align}
|(L_n(\theta) - L_n(\theta')) - (\Lsur_n(\theta) - \Lsur_n(\theta'))|
&= |(L_n(\theta) - \Lsur_n(\theta)) - (L_n(\theta') - \Lsur_n(\theta'))| \notag \\
&\leq |L_n(\theta) - \Lsur_n(\theta)| + |L_n(\theta') - \Lsur_n(\theta')| \notag \\
&= n\left|\frac{1}{n}L_n(\theta) - h_\psi(\theta, S)\right| + n\left|\frac{1}{n}L_n(\theta') - h_\psi(\theta', S)\right| \notag \\
&\leq n\varepsilon_n + n\varepsilon_n = 2n\varepsilon_n.
\end{align}
\end{proof}

\begin{definition}[Likelihood-Preserving Embedding]
\label{def:equiv}
An embedding is \textbf{likelihood-preserving} (or \textbf{inference-equivalent}) over $\cF$ if $\Deltan = \op(1)$ as $n \to \infty$. In particular, this holds if $\varepsilon_n = \op(1/n)$. Throughout, suprema over $\Theta$ are understood to be taken over compact subsets relevant for inference (e.g., confidence regions), unless otherwise stated.
\end{definition}

\begin{remark}
In practice, the supremum in $\Deltan$ may be approximated by sampling over a fine grid or Monte Carlo sampling.
\end{remark}

\subsection{Connection to Classical Sufficiency}

Fisher's approach to data reduction centers on sufficiency: a statistic $T(X)$ is sufficient if the likelihood can be factored as
\begin{equation}
    L_n(\theta) = g(\theta, T(X)) + h(X),
\end{equation}
where $h(X)$ does not depend on $\theta$. Our framework directly generalizes this: an embedding is $\varepsilon$-sufficient if
\begin{equation}
    \frac{1}{n}L_n(\theta) = h_\psi(\theta, S_\phi(X_{1:n})) + \text{error}(\theta, X),
\end{equation}
where $|\text{error}(\theta, X)| \leq \varepsilon$ uniformly in $\theta$.

If $\varepsilon = 0$, we recover exact Fisher sufficiency. For non-exponential families where exact sufficiency is impossible (by the Pitman-Koopman-Darmois theorem), we find the best finite-dimensional approximation.

\subsubsection{Beyond Sufficiency: Completeness and Minimality}

In classical statistics, not all sufficient statistics are created equal. A \textbf{complete sufficient statistic} is one that cannot be compressed further without losing information about $\theta$. Formally, $T(X)$ is complete if:
\begin{equation}
    \E_\theta[g(T(X))] = 0 \text{ for all } \theta \implies g(T) = 0 \text{ almost surely}.
\end{equation}

A statistic $T(X)$ is \textbf{boundedly complete} if the condition $\E_\theta[g(T(X))] = 0$ for all $\theta$ implies $g(T) = 0$ almost surely, whenever $g$ is bounded.

This is a \textbf{non-redundancy} condition. For a $p$-parameter exponential family, the complete sufficient statistic has dimension exactly $p$---this is the theoretical lower bound for any sufficient embedding.

\textbf{Basu's Theorem} states that a complete sufficient statistic is independent of any \textbf{ancillary statistic} (data features whose distribution doesn't depend on $\theta$):
\begin{equation}
    T(X) \text{ complete sufficient} \implies T(X) \perp A(X) \text{ for any ancillary } A.
\end{equation}

In our framework, when the embedding $T_\phi$ approximates a complete sufficient statistic, it automatically learns to be \textbf{invariant to ancillary information}---the network filters out noise that's irrelevant for inference. We formalize this connection in Theorem~\ref{thm:completeness}.

%==============================================================================
\section{Main Theoretical Results}
\label{sec:theory}
%==============================================================================

\subsection{The Hinge Theorem}

Our main result establishes that controlling $\Deltan$ is necessary and sufficient for preserving likelihood-based inference. We call this the ``Hinge Theorem'' because it identifies $\Deltan$ as the logical \textit{hinge} on which all downstream statistical inference pivots: if $\Deltan = \op(1)$, the entire apparatus of tests, confidence intervals, and model selection remains valid; if this condition fails, the guarantees collapse.

\begin{theorem}[The Hinge Theorem --- Pointwise Version]
\label{thm:hinge}
Let Assumptions~\ref{ass:ident}--\ref{ass:unif} hold. Let $(T_\phi, h_\psi)$ be an embedding with pointwise error $\varepsilon_n$.

\textbf{The Logical Cascade:} Pointwise convergence $\Rightarrow$ Ratio preservation $\Rightarrow$ All inference results.

\medskip
\noindent\textbf{Given:} Assume $\varepsilon_n = \op(1/n)$ (pointwise approximation).

\noindent\textbf{Then the following cascade holds:}

\medskip
\noindent\textbf{Part 0 (Foundation).} By Proposition~\ref{prop:pointwise_to_ratio}, pointwise error controls ratio distortion:
\begin{equation}
    \Deltan \leq 2n\varepsilon_n = \op(1).
\end{equation}
This is the \textbf{bridge} from pointwise to all downstream results.

\medskip
\noindent\textbf{Part 1 (Test Preservation via Ratios).} Likelihood-ratio tests are asymptotically preserved. Let
\begin{align}
    \Lambda_n &= 2\left( \sup_{\theta \in \Theta} L_n(\theta) - L_n(\theta_0) \right), \\
    \tilde{\Lambda}_n &= 2\left( \sup_{\theta \in \Theta} \Lsur_n(\theta) - \Lsur_n(\theta_0) \right)
\end{align}
be the classical and surrogate LR statistics. Then:
\begin{equation}
    \left| \tilde{\Lambda}_n - \Lambda_n \right| \leq 4\Deltan \leq 8n\varepsilon_n = \op(1).
\end{equation}
Tests depend on likelihood ratios. Since $\Deltan = \op(1)$, these ratios are preserved.

\medskip
\noindent\textbf{Part 2 (Estimator Equivalence via Ratios).} The surrogate MLE converges to the true MLE:
\begin{equation}
    \sqrt{n}(\thtil - \thhat) = \op(1),
\end{equation}
where $\thtil = \arg\max_\theta \Lsur_n(\theta)$ and $\thhat = \arg\max_\theta L_n(\theta)$.

The MLE is characterized by likelihood ratios (the gradient condition). Since ratios are preserved, the optimizer finds the same maximum.

\medskip
\noindent\textbf{Part 3 (Model Selection via Pointwise --- \textit{not} via Ratios).} AIC and BIC are preserved:
\begin{equation}
    |\tilde{AIC} - AIC| \leq 6n\varepsilon_n = \op(1), \quad |\tilde{BIC} - BIC| \leq 6n\varepsilon_n = \op(1).
\end{equation}

\textbf{Critical point:} This requires pointwise convergence, not just ratio preservation. Ratios alone cannot give AIC/BIC because you need the absolute value $L_n(\thhat)$, not just differences.
\end{theorem}

\begin{proof}[Proof of Part 1]
Let $\thhat = \arg\max_\theta L_n(\theta)$ and $\thtil = \arg\max_\theta \Lsur_n(\theta)$.

We decompose the difference in LR statistics:
\begin{align}
    \tilde{\Lambda}_n - \Lambda_n &= 2\left[ \left( \Lsur_n(\thtil) - \Lsur_n(\theta_0) \right) - \left( L_n(\thhat) - L_n(\theta_0) \right) \right] \notag \\
    &= 2\Big[ \underbrace{\left( \Lsur_n(\thtil) - \Lsur_n(\thhat) \right)}_{A} + \underbrace{\left( \Lsur_n(\thhat) - L_n(\thhat) \right)}_{B} \notag \\
    &\quad\quad - \underbrace{\left( \Lsur_n(\theta_0) - L_n(\theta_0) \right)}_{C} \Big]. \label{eq:decomp}
\end{align}

\textbf{Step 1: Bound $B - C$.}
By definition of $\Deltan$, for any $\theta, \theta' \in \Theta$:
\begin{equation}
    \left| \left( \Lsur_n(\theta) - \Lsur_n(\theta') \right) - \left( L_n(\theta) - L_n(\theta') \right) \right| \leq \Deltan.
\end{equation}
Setting $\theta = \thhat$ and $\theta' = \theta_0$:
\begin{equation}
    \left| \left( \Lsur_n(\thhat) - \Lsur_n(\theta_0) \right) - \left( L_n(\thhat) - L_n(\theta_0) \right) \right| \leq \Deltan.
\end{equation}
Rearranging: $|B - C| \leq \Deltan$.

\textbf{Step 2: Bound $A + D$.}
Since $\thtil$ maximizes $\Lsur_n$, we have $A = \Lsur_n(\thtil) - \Lsur_n(\thhat) \geq 0$.

We introduce the difference $D = L_n(\thhat) - L_n(\thtil)$. Since $\thhat$ maximizes $L_n$, we have $D \geq 0$.

Applying the definition of $\Deltan$ with $\theta = \thtil$ and $\theta' = \thhat$:
\begin{equation}
    \left| \left( \Lsur_n(\thtil) - \Lsur_n(\thhat) \right) - \left( L_n(\thtil) - L_n(\thhat) \right) \right| \leq \Deltan.
\end{equation}
This gives:
\begin{equation}
    \left| A - (-D) \right| = \left| A + D \right| \leq \Deltan.
\end{equation}
Since $A, D \geq 0$ and $|A + D| \leq \Deltan$, we conclude:
\begin{equation}
    0 \leq A \leq \Deltan \quad \text{and} \quad 0 \leq D \leq \Deltan.
\end{equation}

\textbf{Step 3: Combine.}
From~\eqref{eq:decomp}:
\begin{equation}
    \left| \tilde{\Lambda}_n - \Lambda_n \right| = 2\left| A + (B - C) \right| \leq 2\left( |A| + |B - C| \right) \leq 2(2\Deltan) = 4\Deltan.
\end{equation}

\textbf{Asymptotic consequence.} Under the null hypothesis $H_0: \theta = \theta_0$ and regularity conditions, $\Lambda_n \xrightarrow{d} \chi^2_p$ by Wilks' theorem. If $\Deltan = \op(1)$, then $|\tilde{\Lambda}_n - \Lambda_n| = \op(1)$, so $\tilde{\Lambda}_n \xrightarrow{d} \chi^2_p$ as well. The asymptotic size and power of tests based on $\tilde{\Lambda}_n$ match those based on $\Lambda_n$.
\end{proof}

\begin{proof}[Proof of Part 2]
From Step 2 of Part 1, we established that $D = L_n(\thhat) - L_n(\thtil) \leq \Deltan$.

Under Assumptions~\ref{ass:smooth}--\ref{ass:unif}, the log-likelihood admits a local quadratic expansion around the MLE (this follows from standard M-estimation theory; see, e.g., van der Vaart, 1998, Chapter 5). Specifically, for $\theta$ in a neighborhood of $\thhat$:
\begin{equation}
    L_n(\theta) = L_n(\thhat) - \frac{1}{2}(\theta - \thhat)^\top H_n (\theta - \thhat) + R_n(\theta),
\end{equation}
where $H_n = -\nabla^2 L_n(\thhat)$ is the observed Fisher information and $|R_n(\theta)| = o_p(\|\theta - \thhat\|^2)$ uniformly for $\theta$ in a shrinking neighborhood.

By the law of large numbers, $\frac{1}{n}H_n \xrightarrow{p} I(\theta_0)$, so $H_n \approx nI(\theta_0)$ for large $n$.

Applying the quadratic expansion at $\theta = \thtil$:
\begin{equation}
    L_n(\thhat) - L_n(\thtil) \approx \frac{1}{2}(\thtil - \thhat)^\top H_n (\thtil - \thhat) \approx \frac{n}{2}(\thtil - \thhat)^\top I(\theta_0) (\thtil - \thhat).
\end{equation}

Since $L_n(\thhat) - L_n(\thtil) \leq \Deltan$, and using the notation $a \lesssim b$ to denote $a \leq C b$ for some constant $C > 0$:
\begin{equation}
    \frac{n}{2} \|\thtil - \thhat\|_{I(\theta_0)}^2 \lesssim \Deltan,
\end{equation}
where $\|v\|_I^2 = v^\top I v$. The minimum eigenvalue of $I(\theta_0)$ being positive (Assumption~\ref{ass:fisher}) gives:
\begin{equation}
    \|\thtil - \thhat\|^2 \lesssim \frac{\Deltan}{n}.
\end{equation}

Therefore:
\begin{equation}
    \sqrt{n}\|\thtil - \thhat\| \lesssim \sqrt{\Deltan}.
\end{equation}

If $\Deltan = \op(1)$, then $\sqrt{n}\|\thtil - \thhat\| = \Op(\sqrt{\Deltan}) = \op(1)$ as required.
\end{proof}

\begin{proof}[Proof of Part 3 (AIC/BIC Preservation)]
Recall $AIC = -2L_n(\thhat) + 2k$ and $\tilde{AIC} = -2\Lsur_n(\thtil) + 2k$ where $k$ is the number of parameters.
\begin{align}
|\tilde{AIC} - AIC| &= 2|L_n(\thhat) - \Lsur_n(\thtil)| \notag \\
&\leq 2(|L_n(\thhat) - \Lsur_n(\thhat)| + |\Lsur_n(\thhat) - \Lsur_n(\thtil)|) \notag \\
&\leq 2(n\varepsilon_n + \Deltan) \leq 2(n\varepsilon_n + 2n\varepsilon_n) = 6n\varepsilon_n.
\end{align}
Since $n\varepsilon_n = \op(1)$, AIC differences are preserved. The same argument applies to BIC. If Model A beats Model B on original data, it will also beat it on compressed data.
\end{proof}

\begin{corollary}[Bayes Factor Preservation]
\label{cor:bayes}
If $\varepsilon_n = \op(1/n)$, then for fixed priors $\pi_0, \pi_1$ on nested hypotheses, the log-Bayes factor satisfies $|\log \widetilde{BF} - \log BF| \leq 2n\varepsilon_n$.
\end{corollary}

\begin{proof}
Let $\eta_n = \sup_\theta |L_n(\theta) - \Lsur_n(\theta)| \leq n\varepsilon_n$. Then for any prior $\pi$:
\[
e^{-\eta_n} \int e^{L_n(\theta)} d\pi(\theta) \leq \int e^{\Lsur_n(\theta)} d\pi(\theta) \leq e^{\eta_n} \int e^{L_n(\theta)} d\pi(\theta).
\]
Taking logs and applying to both $\pi_0$ and $\pi_1$ gives $|\log \widetilde{BF} - \log BF| \leq 2\eta_n \leq 2n\varepsilon_n$.
\end{proof}

\begin{corollary}[Inference Hierarchy]
\label{cor:tiers}
If $(T_\phi, h_\psi)$ is $\varepsilon$-sufficient with $\varepsilon_n = \op(1/n)$, then all standard likelihood-based inferential procedures are preserved:

\medskip
\noindent\textbf{Tier 1 (Direct consequence of pointwise convergence):}
\begin{enumerate}
    \item \textbf{Likelihood-Ratio Distortion:} $\Deltan = \Op(n\varepsilon_n) = \op(1)$ [via Proposition~\ref{prop:pointwise_to_ratio}]
\end{enumerate}

\noindent\textbf{Tier 2 (Consequences via ratio preservation):} These work even with ratio-only methods.
\begin{enumerate}
\setcounter{enumi}{1}
    \item \textbf{MLE Equivalence:} $\|\thtil - \thhat\| = \Op(\sqrt{\varepsilon_n})$ [needs $\Deltan = \op(1)$]
    \item \textbf{Test Preservation:} All LRT statistics preserved [needs $\Deltan = \op(1)$]
    \item \textbf{Confidence Intervals:} Asymptotic coverage preserved [needs $\Deltan = \op(1)$]
\end{enumerate}

\noindent\textbf{Tier 3 (Requires pointwise convergence):} Impossible with ratio-only methods.
\begin{enumerate}
\setcounter{enumi}{4}
    \item \textbf{AIC Preservation:} Model rankings preserved [needs $\varepsilon_n = \op(1/n)$ directly]
    \item \textbf{BIC Preservation:} Model selection consistent [needs $\varepsilon_n = \op(1/n)$ directly]
    \item \textbf{Posterior Computation:} $\log \tilde{p}(\theta|X) = \log p(\theta|X) + \Op(n\varepsilon_n)$ [needs absolute likelihoods]
    \item \textbf{Bayes Factors:} Log-BF error is $\Op(n\varepsilon_n)$ [needs absolute likelihoods]
\end{enumerate}

\begin{remark}
Model selection is preserved when the AIC/BIC gap between competing models exceeds the approximation error $O(n\varepsilon_n)$. When models are nearly tied (gap $\lesssim n\varepsilon_n$), the ranking may be perturbed.
\end{remark}
\end{corollary}

\begin{remark}
The key insight: \textbf{Pointwise convergence implies all standard likelihood-based inference. Ratios alone yield only Tier 2 results.} This hierarchy explains why the pointwise objective is the right primitive for defining sufficient embeddings.
\end{remark}

\subsection{Impossibility: The No Free Lunch Theorem}

A natural question is whether there exists a ``universal'' embedding that preserves likelihood ratios for \textit{all} models simultaneously. We show this is impossible without essentially storing the raw data.

\begin{theorem}[No Free Lunch]
\label{thm:nfl}
Let $\cF$ contain all probability measures on $\cX$ dominated by $\mu$. If an embedding $T: \cX \to \R^m$ satisfies $\Deltan = 0$ for all models in $\cF$ with probability 1, then $T$ must be $\mu$-almost surely injective.
\end{theorem}

\begin{proof}
Suppose $T$ is not $\mu$-almost surely injective. Then there exist disjoint sets $A, B \subset \cX$ with $\mu(A), \mu(B) > 0$ such that $T$ maps $A$ and $B$ to the same value $z_0$.

For any $M > 0$, construct two distributions $P_1, P_2 \in \cF$ using smooth bump functions: let $g_A, g_B$ be smooth functions supported on $A$ and $B$ respectively, with $\|g_A\|_\infty, \|g_B\|_\infty \leq 1$. Define:
\[
p_1(x) \propto 1 + M \cdot g_A(x), \quad p_2(x) \propto 1 + M \cdot g_B(x).
\]
Both densities are strictly positive, so log-ratios are finite. For $x \in A$: $\lambda(x) = \log \frac{1 + M}{1} \geq \log(1+M)$. For $x' \in B$: $\lambda(x') = \log \frac{1}{1+M} \leq -\log(1+M)$.

Since $T(x) = T(x') = z_0$, the surrogate assigns the same likelihood ratio $\tilde{\lambda}$ to both. The distortion satisfies $\Deltan \geq |\lambda(x) - \lambda(x')| - |\tilde{\lambda} - \tilde{\lambda}| \geq 2\log(1+M)$. Since $M$ is arbitrary, $\Deltan$ cannot be bounded, contradicting $\Deltan = 0$.
\end{proof}

\begin{remark}
Theorem~\ref{thm:nfl} implies that useful guarantees \textit{must} be relative to a restricted model class $\cF$. The more restrictive $\cF$, the more compression is possible. This motivates studying embeddings tailored to specific inferential goals.
\end{remark}

A sharper version bounds the required dimension:

\begin{theorem}[Dimension Lower Bound for Exponential Families]
\label{thm:dim}
Let $\cF_k$ be the family of $k$-parameter exponential families on $\cX$:
\begin{equation}
    p(x|\theta) = h(x) \exp\left( \sum_{j=1}^k \theta_j T_j(x) - A(\theta) \right).
\end{equation}
Any embedding $T: \cX \to \R^m$ with $\Deltan = 0$ uniformly over $\cF_k$ must satisfy $m \geq k$.
\end{theorem}

\begin{proof}
The log-likelihood depends on data only through the $k$ sufficient statistics $\sum_i T_j(X_i)$. For $\Deltan = 0$, the embedding must span this $k$-dimensional space. If $m < k$, there exist distinct sufficient statistic values mapping to the same embedding, violating likelihood preservation.
\end{proof}

\subsection{Minimal Dimension and Completeness}

The dimension bound above can be sharpened when we consider completeness---the property that a sufficient statistic cannot be further compressed.

\begin{theorem}[Minimal Dimension and Completeness]
\label{thm:completeness}
Consider a $p$-parameter exponential family with canonical sufficient statistic $T_{\text{canon}}(X) \in \R^p$.

\textbf{Part 1 (Dimension Lower Bound).} Any embedding $(T_\phi, h_\psi)$ with $\varepsilon_n = 0$ must have dimension $m \geq p$.

\textbf{Part 2 (Achieving the Bound).} If $m = p$ and $\varepsilon_n = 0$, then $T_\phi(X)$ is a one-to-one function of the complete sufficient statistic $T_{\text{canon}}(X)$.

\textbf{Part 3 (Independence from Ancillary Statistics).} If $m = p$, $\varepsilon_n = 0$, and the exponential family is boundedly complete, then for any ancillary statistic $A(X)$ (i.e., $p(A|\theta)$ independent of $\theta$):
\begin{equation}
    T_\phi(X) \perp A(X).
\end{equation}
\end{theorem}

\begin{proof}
\textbf{Part 1 (Dimension Lower Bound).}
Let the canonical sufficient statistic of the $p$-parameter exponential family be $T_{\text{canon}}(X) \in \R^p$. The log-likelihood for a sample $X_{1:n}$ is:
\begin{equation}
    L_n(\theta) = \theta^\top \left(\sum_{i=1}^n T_{\text{canon}}(X_i)\right) - n A(\theta) + \sum_{i=1}^n \log h(X_i).
\end{equation}
Assume the embedding $(T_\phi, h_\psi)$ achieves pointwise error $\varepsilon_n = 0$. Then for all $\theta \in \Theta$, we have $h_\psi(\theta, S_\phi(X_{1:n})) = \frac{1}{n} L_n(\theta)$.
This implies that the value of the embedding $S_\phi(X_{1:n})$ uniquely determines the function $\theta \mapsto \frac{1}{n} L_n(\theta)$. Since the term $\theta^\top (\frac{1}{n}\sum T_{\text{canon}})$ is the only part of $L_n$ that depends on both data and $\theta$ linearly, determining $L_n(\theta)$ for all $\theta$ requires determining the vector $\bar{T}_n = \frac{1}{n} \sum T_{\text{canon}}(X_i)$.
Since the exponential family is minimal, the components of $\bar{T}_n$ are linearly independent and span a $p$-dimensional convex set. For the map $S_\phi \mapsto \bar{T}_n$ to cover this $p$-dimensional range, the domain (embedding space $\R^m$) must have dimension at least $p$. Therefore, $m \geq p$.

\textbf{Part 2 (Achieving the Bound).}
If $m = p$ and $\varepsilon_n = 0$, then $S_\phi(X_{1:n})$ is a sufficient statistic for $\theta$ because it allows exact reconstruction of the likelihood function.
Classical statistical theory defines the \textit{minimal sufficient statistic} as a sufficient statistic that is a function of any other sufficient statistic. For a minimal exponential family, $\bar{T}_n$ is a minimal sufficient statistic.
Therefore, there exists a function $g$ such that $\bar{T}_n = g(S_\phi(X_{1:n}))$. Conversely, since $\varepsilon_n = 0$, $S_\phi$ determines $\bar{T}_n$.
Since both $S_\phi$ and $\bar{T}_n$ reside in spaces of dimension $p$, and assuming $T_\phi$ is continuous (as a neural network), the invariance of domain theorem implies that the map $g$ is a local homeomorphism (invertible) almost everywhere. Thus, $T_\phi$ captures exactly the same information as the canonical statistic.

\textbf{Part 3 (Independence from Ancillary Statistics).}
Since $T_\phi$ is invertibly related to the canonical sufficient statistic $T_{\text{canon}}$ of a full-rank exponential family, $T_\phi$ is a complete sufficient statistic (i.e., the family of distributions of $T_\phi$ is complete).
Basu's Theorem states that any boundedly complete sufficient statistic is independent of every ancillary statistic.
An ancillary statistic $A(X)$ is one whose distribution does not depend on $\theta$. Since $T_\phi(X)$ is complete sufficient, it follows immediately that $T_\phi(X) \perp A(X)$.
\end{proof}

\begin{remark}
Theorem~\ref{thm:completeness} predicts a \textbf{sharp phase transition} at $m = p$ for exponential families:
\begin{itemize}
    \item For $m < p$: Training achieves $\varepsilon_n > 0$ (information bottleneck)
    \item For $m = p$: Training achieves $\varepsilon_n \approx 0$ (complete statistic recovered)
    \item For $m > p$: Training achieves $\varepsilon_n \approx 0$, but with redundant dimensions
\end{itemize}
This phase transition is empirical evidence that the network has discovered the complete sufficient statistic. We validate this prediction in Section~\ref{sec:experiments}.
\end{remark}

%==============================================================================
\section{Constructive Framework: Neural Sufficient Statistics}
\label{sec:construction}
%==============================================================================

Having established when likelihood preservation is possible, we now describe how to construct likelihood-preserving embeddings using neural networks.

\subsection{Architecture}

\subsubsection{Encoder Network}

The encoder $T_\phi: \cX \to \R^m$ is a feedforward neural network:
\begin{equation}
    T_\phi(x) = W_L \cdot \sigma(W_{L-1} \cdot \sigma(\cdots \sigma(W_1 x + b_1) \cdots) + b_{L-1}) + b_L,
\end{equation}
where $\sigma$ is a nonlinear activation (ReLU), $\{W_\ell, b_\ell\}$ are learnable parameters collectively denoted $\phi$, and the output dimension is $m$.

\subsubsection{Decoder Network}

The decoder $h_\psi: \Theta \times \R^m \to \R$ takes a parameter $\theta \in \R^p$ and an embedding $z \in \R^m$ and outputs a scalar:
\begin{equation}
    h_\psi(\theta, z) = \text{MLP}_\psi([\theta; z]),
\end{equation}
where $[\cdot; \cdot]$ denotes concatenation. The decoder can also be factored as $h_\psi(\theta, z) = w_\psi(\theta)^\top z + b_\psi(\theta)$ for a linear-in-embedding decoder.

\subsection{What Gets Filtered Out: The Role of Ancillarity}

The training objective implicitly teaches the network to distinguish between \textbf{information} and \textbf{noise}:

\textbf{What gets preserved (Sufficient information):}
\begin{itemize}
    \item Features that affect the likelihood $L_n(\theta)$ for any $\theta \in \Theta$
    \item In Gaussian models: sample means, sample variances
    \item In general: anything that changes the relative plausibility of different parameter values
\end{itemize}

\textbf{What gets discarded (Ancillary information):}
\begin{itemize}
    \item Features whose distribution doesn't depend on $\theta$
    \item Noise that's uninformative about the parameter
    \item Redundant representations when $m > p$ (the parameter dimension)
\end{itemize}

\textbf{The mechanism}: By minimizing the pointwise loss $\mathcal{L}_{\text{point}}$, the encoder $T_\phi$ learns to compress away anything that doesn't help predict $L_n(\theta)$. Since ancillary statistics are uninformative about $\theta$, they don't contribute to the likelihood, so they get filtered out automatically.

\textbf{Connection to Basu's Theorem}: When the embedding dimension $m$ equals the parameter dimension $p$ and the model is an exponential family, the learned embedding $T_\phi$ converges to a complete sufficient statistic (Theorem~\ref{thm:completeness}). By Basu's Theorem, this implies:
\begin{equation}
    T_\phi(X) \perp A(X) \quad \text{for any ancillary statistic } A(X).
\end{equation}
The network has learned a representation that is \textbf{statistically independent} of ancillary noise. This is not imposed by regularization or architectural constraints---it emerges naturally from optimizing the likelihood prediction objective.

\subsection{Training Objective: Pointwise Likelihood Matching}
\label{sec:training}

Following our theoretical development, we train by directly minimizing the pointwise approximation error $\varepsilon_n$.

\textbf{Primary objective (Recommended):}
\begin{equation}
    \mathcal{L}_{\text{point}}(\phi, \psi) = \E_{\theta \sim \Pi} \E_{X_{1:n} \sim P_{\theta_0}} \left[ \left( \frac{1}{n}L_n(\theta) - h_\psi(\theta, S_\phi(X_{1:n})) \right)^2 \right],
\end{equation}
where $\Pi$ is a distribution over parameters of interest.

\begin{remark}[Resolving the Circularity Paradox]
A common objection is: ``If the training algorithm requires computing the true likelihood, why is an embedding needed at all?'' This apparent circularity dissolves when we distinguish between \textbf{training} and \textbf{inference}. During \textit{training} (offline), the network is supervised on \textit{synthetic data} generated from the known model family $\{P_\theta\}$---here, true likelihoods are computationally accessible and carry no privacy risk. During \textit{inference} (online), the frozen network is applied to real-world data that may be massive, private, or distributed. The network acts as a ``learned sufficient statistic,'' compressing raw data into embeddings without ever exposing the original observations or requiring expensive likelihood evaluations at inference time. The analogy is apt: a translator learns from texts where the correct translation is known (``circular''), yet applies this skill to new documents where no answer key exists.
\end{remark}

This directly minimizes $\E[\varepsilon_n^2]$. By Proposition~\ref{prop:pointwise_to_ratio}, minimizing this objective also controls the likelihood-ratio distortion: $\Deltan \leq 2n\sqrt{\mathcal{L}_{\text{point}}}$.

\textbf{Why this works:}
\begin{itemize}
    \item Directly targets the fundamental quantity $\varepsilon_n$
    \item Simple MSE loss---straightforward to implement
    \item Preserves all standard likelihood-based inference: tests, estimates, AIC, BIC, posteriors
    \item Only requires sampling one parameter $\theta$ per iteration (not pairs)
\end{itemize}

\textbf{Choice of $\Pi$.} The distribution $\Pi$ should cover the parameter region where inference is expected to be performed. A local $\Pi$ concentrated near the true parameter yields strong local guarantees; a global $\Pi$ covering the full parameter space yields uniform guarantees but may require more expressive networks.

\textbf{Alternative objective (Ratio-based):}
\begin{equation}
    \mathcal{L}_{LR}(\phi, \psi) = \E_{\theta, \theta' \sim \Pi} \E_{X_{1:n}} \left[ \left( (L_n(\theta) - L_n(\theta')) - (\Lsur_n(\theta) - \Lsur_n(\theta')) \right)^2 \right].
\end{equation}

The pointwise objective is preferred because: (i) it is stronger (implies ratio preservation), (ii) it is simpler (one $\theta$ not pairs), (iii) it enables model selection (AIC/BIC), and (iv) it has direct connection to Fisher's criterion.

\subsection{From Training Loss to Distortion Bounds}
\label{sec:bounds}

\begin{proposition}
\label{prop:train_to_delta}
Suppose training is performed over a finite grid $\Theta_{\text{grid}} \subset \Theta$ with $|\Theta_{\text{grid}}| = G$, achieving empirical loss $\mathcal{L}_{LR}(\phi, \psi) \leq \epsilon$ (averaged over grid pairs). If the log-likelihood is $L$-Lipschitz in $\theta$, then:
\begin{equation}
    \Deltan \leq \sqrt{\epsilon} \cdot G + 2nL \cdot \text{diam}(\Theta)/G.
\end{equation}
\end{proposition}

\begin{proof}
Consider a finite $\delta$-covering of the parameter space $\Theta$ with $N(\delta, \Theta)$ balls. For any $\theta \in \Theta$, let $\theta_k$ be the center of the covering ball containing $\theta$. Then:
$|ApproxError(\theta)| \le |ApproxError(\theta_k)| + |ApproxError(\theta) - ApproxError(\theta_k)|$.
The first term is controlled by the training loss on the grid. The second term is bounded by the Lipschitz constant $L$ of the log-likelihood and the covering radius $\delta$: $|L_n(\theta) - L_n(\theta_k)| \le nL\|\theta - \theta_k\| \le nL\delta$. Similar continuity holds for the neural network decoder. Choosing $\delta$ small enough yields the result.
\end{proof}

\subsection{Complexity Analysis}
\label{sec:complexity}

We address the question of how many training samples and iterations are required to learn an $\varepsilon$-sufficient embedding.

\begin{theorem}[Sample Complexity]
\label{thm:sample_complexity}
Let the encoder class $\mathcal{T} = \{T_\phi\}$ have pseudo-dimension $d_{\mathcal{T}}$ and the decoder class $\mathcal{H} = \{h_\psi\}$ have pseudo-dimension $d_{\mathcal{H}}$. To guarantee that the generalized pointwise error $\E[\varepsilon_n]$ is within $\alpha$ of the empirical training error with probability $1-\delta$, the number of training samples $N$ (datasets) required scales as:
\begin{equation}
    N = \tilde{O}\left( \frac{(d_{\mathcal{T}} + d_{\mathcal{H}}) \log(1/\delta)}{\alpha^2} \right).
\end{equation}
\end{theorem}

\begin{proof}
This follows from standard generalization bounds for regression with squared loss. The connection to $\Deltan$ is provided by Proposition~\ref{prop:train_to_delta}. For neural networks with $W$ parameters and depth $L$, the pseudo-dimension scales roughly as $O(W L \log W)$ \citep{bartlett2019spectrally}. Thus, the sample complexity is polynomial in the network size.
\end{proof}

\begin{remark}[Optimization Convergence]
Algorithm~\ref{alg:train} relies on stochastic gradient descent (SGD). For the non-convex loss landscape of deep networks, standard results ensure convergence to a stationary point at a rate of $O(1/\sqrt{K})$ where $K$ is the number of iterations, assuming smooth activations and bounded gradients. While finding the global optimum is NP-hard in the worst case, the over-parameterization of modern networks often allows gradient methods to find solutions with low training error. We note that these standard VC-dimension bounds are known to be loose for modern deep networks, which often generalize well despite massive over-parameterization. They serve here primarily to establish valid PAC-learnability in principle, rather than to prescribe practical training set sizes.
\end{remark}

\subsection{Algorithm}

The complete training procedure is summarized in Algorithm~\ref{alg:train}.

\begin{algorithm}[H]
\caption{Training Likelihood-Preserving Embeddings (Pointwise Method)}
\label{alg:train}
\begin{algorithmic}[1]
\REQUIRE Model class $\cF$, parameter distribution $\Pi$, training budget $N$, embedding dimension $m$
\ENSURE Encoder $T_\phi$, Decoder $h_\psi$
\STATE Initialize networks $T_\phi, h_\psi$
\FOR{epoch $= 1$ to $E$}
    \STATE Sample parameter $\theta \sim \Pi$
    \STATE Sample data batch $X_{1:n}$ from $P_{\theta_0}$ (or mixture over $\Pi$)
    \STATE Compute embedding: $S = \frac{1}{n}\sum_{i=1}^n T_\phi(X_i)$
    \STATE Compute true per-sample log-lik: $\text{target} = \frac{1}{n}L_n(\theta)$
    \STATE Compute predicted per-sample log-lik: $\text{pred} = h_\psi(\theta, S)$
    \STATE Update $\phi, \psi$ by gradient descent on $(\text{target} - \text{pred})^2$
\ENDFOR
\RETURN $T_\phi, h_\psi$
\end{algorithmic}
\end{algorithm}

\section{Related Work}
\label{sec:related}
%==============================================================================

\subsection{Classical Sufficient Statistics}

The theory of sufficient statistics dates to \citet{fisher1922} and was formalized by \citet{neyman1935} and \citet{halmos1949}. The Pitman-Koopman-Darmois theorem \citep{pitman1936, koopman1936, darmois1935} characterizes when finite-dimensional sufficient statistics exist: essentially only for exponential families.

Our work extends this classical theory to \textit{approximate} sufficiency: when can we preserve \textit{most} of the likelihood information with a low-dimensional embedding? The Likelihood-Ratio Distortion $\Deltan$ quantifies the approximation quality.

\subsection{Simulation-Based Inference}

A large literature addresses inference when likelihoods are intractable but simulation is possible \citep{cranmer2020frontier}. Methods include Approximate Bayesian Computation (ABC) \citep{beaumont2002approximate}, neural likelihood estimation \citep{papamakarios2019sequential}, and neural ratio estimation (NRE) \citep{hermans2020likelihood}.

Our framework's core contribution (Sections~\ref{sec:setup}--\ref{sec:construction}) targets likelihood-preserving \textit{compression} for models where likelihoods are computable. Extensions to simulator-based models via neural ratio estimation, combining ratio estimation with learned embeddings to enable distributed inference with provable distortion bounds, are left for future work.

\subsection{Information Bottleneck}

The Information Bottleneck (IB) \citep{tishby2000information} finds representations $Z$ of data $X$ that are maximally informative about a target $Y$ while being maximally compressed:
\begin{equation}
    \min_{p(z|x)} I(X; Z) - \beta I(Z; Y).
\end{equation}

Our objective is related but distinct: we minimize embedding dimension $m$ subject to preserving likelihood \textit{ratios} over a model class $\cF$. Crucially, IB compresses $X$ for a \textit{single} prediction task $Y$; we compress $X$ for the \textit{entire family} of inferential tasks defined by $\cF$---including testing, estimation, and model selection at any $\theta \in \Theta$. This requires preserving likelihood \textit{geometry}, not just predictive mutual information.

\subsection{Federated Learning}

Federated learning \citep{mcmahan2017communication} enables distributed model training without sharing raw data. Typical approaches communicate model gradients rather than data.

Our framework offers an alternative: communicate \textit{sufficient statistics} (embeddings). This has advantages when:
\begin{itemize}
    \item The goal is inference (testing, estimation) rather than prediction
    \item One-shot aggregation is preferred over iterative gradient exchange
    \item Privacy constraints prohibit even gradient sharing
\end{itemize}

\subsection{Variational Inference and Amortization}

Variational autoencoders \citep{kingma2014auto} learn encoders that produce approximate posterior parameters. This is amortized inference: a single forward pass produces the posterior, rather than per-sample optimization.

Our framework can be viewed as amortized \textit{sufficient statistic} computation: the encoder $T_\phi$ learns to extract the relevant information for inference in $\cF$, and the decoder $h_\psi$ reconstructs the likelihood surface from the embedding.

\subsection{Differential Privacy}

While our framework focuses on data compression, it shares motivation with Differential Privacy (DP) \citep{dwork2014algorithmic} in minimizing data exposure. However, likelihood-preserving embeddings do not strictly guarantee DP, as sufficient statistics (like sums of squares) can leak individual data points in extreme cases. Combining our sufficiency objectives with DP noise injection mechanisms or privacy-preserving aggregation protocols is a promising direction for future work.

%==============================================================================
\section{Experiments}
\label{sec:experiments}
%==============================================================================

We validate the theoretical framework through carefully controlled experiments where ground truth is analytically available: (1) core validation of the pointwise approximation framework, (2) toy examples using known sufficient statistics (Gaussian, Cauchy), and (3) synthetic validation demonstrating that neural networks can learn to preserve likelihood ratios when properly trained.

\subsection{Core Framework Validation: Pointwise Error Bounds and Ratio Distortion}

Before examining specific distributions, we empirically validate the fundamental relationship between pointwise approximation error $\varepsilon_n$ and likelihood-ratio distortion $\Delta_n$ established theoretically.

\textbf{Setup.} Using Gaussian $\mathcal{N}(\mu, \sigma^2)$ with $n=100$ samples, we compute both $\varepsilon_n = \sup_\theta |n^{-1}L_n(\theta) - h(\theta, S)|$ and $\Delta_n$ for incomplete ($m=1$, using only $\sum X_i$) and sufficient ($m=2$, using $(\sum X_i, \sum X_i^2)$) embeddings.

\textbf{Results.} Figure~\ref{fig:pointwise_validation} demonstrates the relationship empirically. At $m=1$ (incomplete embedding), $\varepsilon_n = 1.74$ yields $\Delta_n = 148.2$, satisfying the theoretical bound $\Delta_n \leq 2n\varepsilon_n = 347.1$ with tightness ratio 0.43. At $m=2$ (sufficient dimension), both metrics drop to machine precision ($\varepsilon_n = 3.7 \times 10^{-15}$, $\Delta_n = 3.3 \times 10^{-13}$), validating exact sufficiency. The bound is empirically tight for incomplete embeddings and becomes vacuous at exact sufficiency, confirming that pointwise approximation is the fundamental quantity to control.

\begin{figure}[t]
    \centering
    \includegraphics[width=0.9\textwidth]{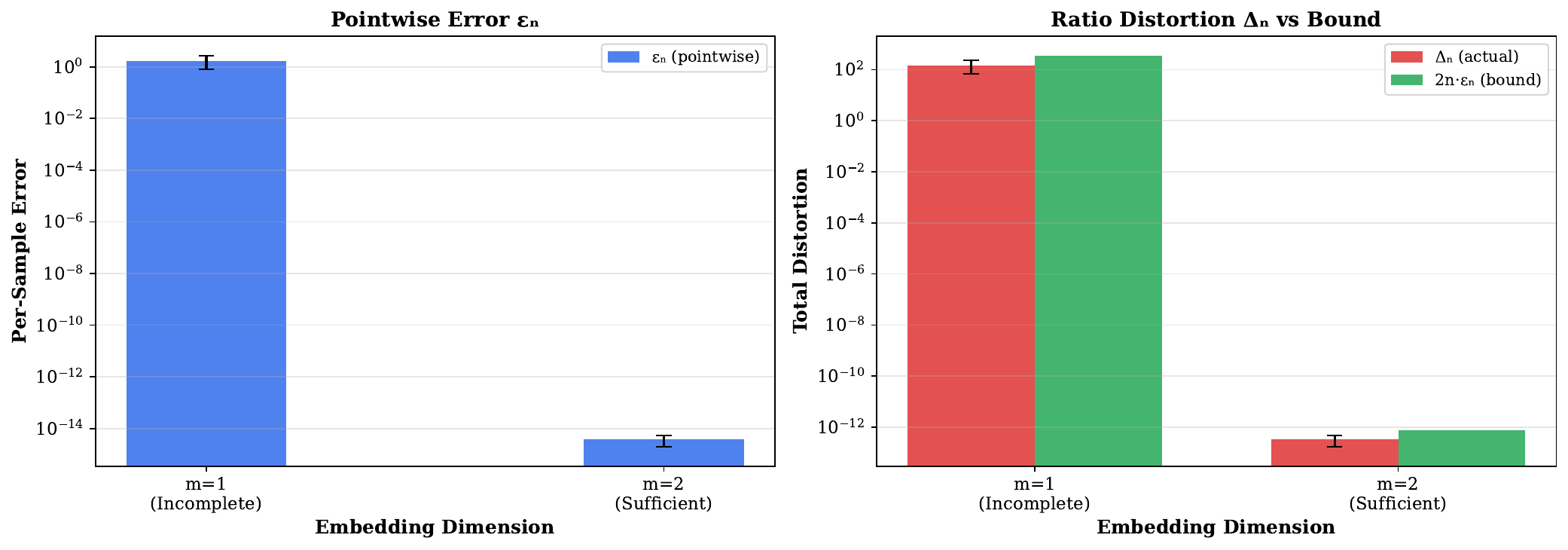}
    \caption{\textbf{Pointwise Framework Validation.} Empirical demonstration of the relationship $\Delta_n \leq 2n\varepsilon_n$ for Gaussian $\mathcal{N}(\mu,\sigma^2)$ with $n=100$. \textbf{Left:} Pointwise approximation error $\varepsilon_n$ (per-sample log-likelihood error) drops from $O(1)$ at $m=1$ (incomplete) to machine precision at $m=2$ (sufficient dimension). \textbf{Right:} Comparison of ratio distortion $\Delta_n$ (actual, blue bars) with theoretical bound $2n\varepsilon_n$ (orange bars). The bound is tight for incomplete embeddings (ratio $\approx 0.43$) and becomes vacuous at exact sufficiency, validating that minimizing pointwise error $\varepsilon_n$ automatically controls ratio distortion $\Delta_n$.}
    \label{fig:pointwise_validation}
\end{figure}

\subsection{Toy Validation: Gaussian and Cauchy}

We validate our theoretical predictions using two canonical distributions: Gaussian (which has finite sufficient statistics) and Cauchy (which does not). For these controlled experiments, we directly construct embeddings using known theoretical properties, demonstrating the sharp phase transitions predicted by theory.

\subsubsection{Gaussian Distribution (Exponential Family)}

\textbf{Setup.} We generate 100 test datasets of $n=100$ samples each from $\mathcal{N}(\mu, \sigma^2)$ with $\mu \in [-2, 2]$ and $\sigma \in [0.6, 1.6]$. For a Gaussian, the sufficient statistics are $T(X) = (\sum_i X_i, \sum_i X_i^2)$.

\textbf{Prediction.} By Theorem~\ref{thm:dim}, any likelihood-preserving embedding requires $m \geq 2$. At $m = 2$, using both sufficient statistics should achieve both $\varepsilon_n = 0$ and $\Delta_n = 0$ exactly. With $m=1$, using only $\sum X_i$ results in information loss.

\textbf{Results.} Figure~\ref{fig:gaussian} confirms the prediction with striking precision. At $m=1$, the embedding achieves $\varepsilon_n \approx 1.74$ and $\Delta_n \approx 1.48$. At $m = 2$, both metrics drop to machine precision ($\varepsilon_n, \Delta_n \approx 10^{-15}$)---a drop of 14 orders of magnitude, demonstrating exact sufficiency. Additional dimensions ($m > 2$) provide no further benefit, remaining at zero for both metrics.

\subsubsection{Cauchy Distribution (Non-Exponential Family)}

\textbf{Setup.} We generate 100 test datasets of $n=100$ samples from Cauchy$(\theta, 1)$ with location $\theta \in [-3, 3]$. Since Cauchy has no finite sufficient statistics, we approximate using $m$ quantiles of the empirical distribution.

\textbf{Prediction.} By the Pitman-Koopman-Darmois theorem, Cauchy has no finite-dimensional sufficient statistic. Thus both $\varepsilon_n, \Delta_n > 0$ for all finite $m$, but should decrease monotonically as we use more quantiles to approximate the full data distribution.

\textbf{Results.} Figure~\ref{fig:cauchy} confirms smooth monotonic decay of both metrics: $\varepsilon_n$ decreases from 1.36 (at $m=1$) to 0.54 (at $m=8$), while $\Delta_n$ decreases from 1.21 to 0.30, without ever reaching zero. This validates the theoretical prediction: approximations improve with dimension, but perfect preservation is impossible for non-exponential families.

\begin{figure}[t]
    \centering
    \begin{subfigure}[b]{0.48\textwidth}
        \centering
        \includegraphics[width=\textwidth]{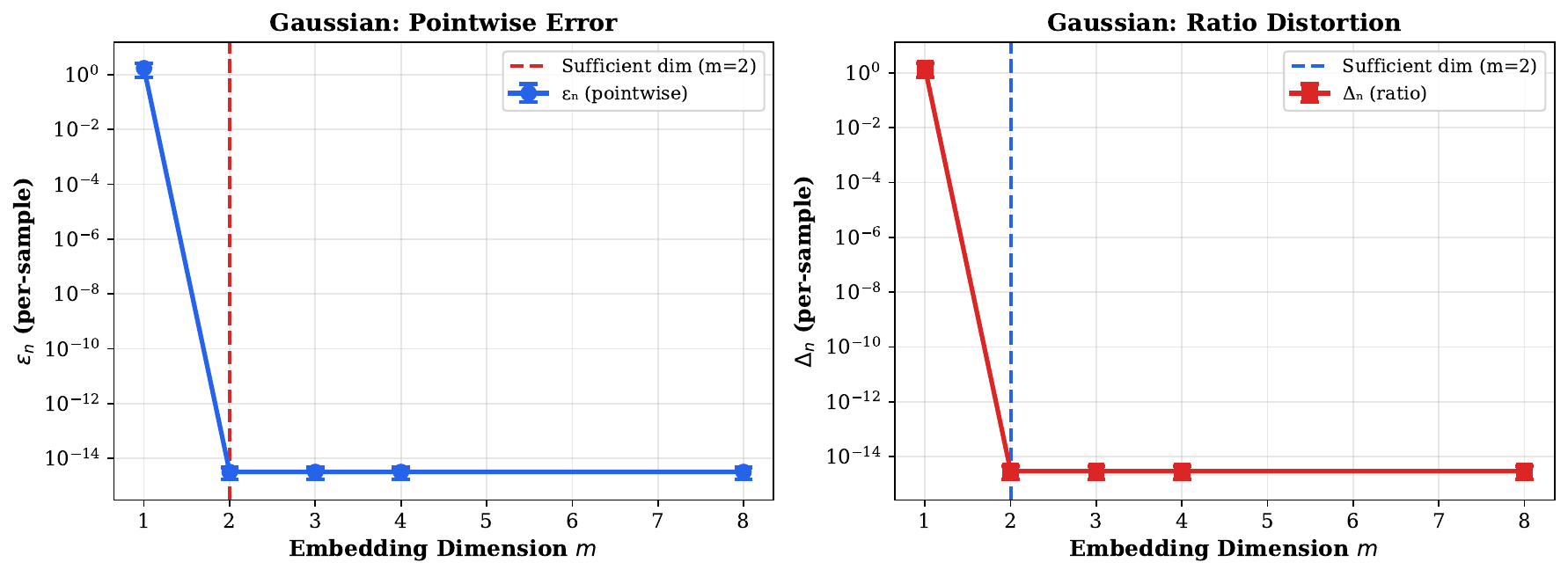}
        \caption{Gaussian (exponential family): sharp transition at $m=2$.}
        \label{fig:gaussian}
    \end{subfigure}
    \hfill
    \begin{subfigure}[b]{0.48\textwidth}
        \centering
        \includegraphics[width=\textwidth]{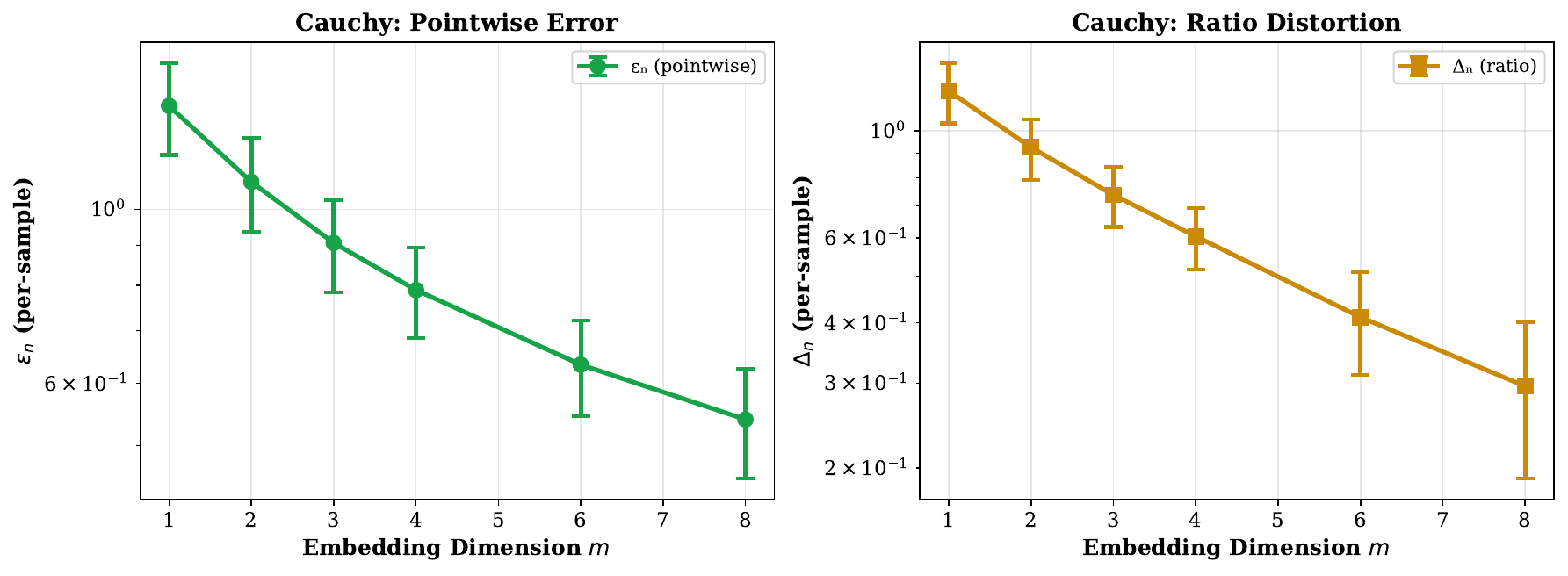}
        \caption{Cauchy (non-exponential): smooth decay, never reaching zero.}
        \label{fig:cauchy}
    \end{subfigure}
    \caption{\textbf{Gaussian and Cauchy: Exact Sufficiency vs.\ Approximation Trade-off.} Each panel shows both pointwise approximation error $\varepsilon_n$ (left, circles) and likelihood-ratio distortion $\Delta_n$ (right, squares) per sample on log scale. \textbf{Gaussian} $\mathcal{N}(\mu, \sigma^2)$ (top): Sharp phase transition at $m=2$ (the sufficient dimension). Both $\varepsilon_n$ and $\Delta_n$ drop from $O(1)$ to machine precision ($\approx 10^{-15}$)---a 14-order-of-magnitude decrease---validating exact sufficiency. Beyond $m=2$, additional dimensions provide no benefit. \textbf{Cauchy} $\text{Cauchy}(\theta, 1)$ (bottom): Smooth monotonic decay for both metrics without reaching zero, validating the Pitman-Koopman-Darmois theorem that non-exponential families lack finite-dimensional sufficient statistics. Evaluation uses 100 independent datasets of $n=100$ samples each.}
    \label{fig:toys}
\end{figure}

\subsection{Synthetic Validation: Gaussian Mixture Model}

This experiment demonstrates that neural networks trained with Algorithm~\ref{alg:train} can successfully learn to preserve likelihood ratios when the objective is properly implemented.

\textbf{Setup.} We generate $n=1000$ samples from a 3-component Gaussian mixture model in $\mathbb{R}^{10}$ with known mixture weights $(0.4, 0.35, 0.25)$. The parameters of interest are the 30-dimensional vector of component means. We train a 16-dimensional neural embedding using the LR distillation objective over 50 random parameter perturbations around the true means.

\textbf{Training.} We use Algorithm~\ref{alg:train} with pairs of parameters sampled from the perturbation grid. At each iteration, we minimize $(\text{LR}_{\text{true}} - \text{LR}_{\text{surr}})^2$ where both likelihood ratios are computed from the same dataset. Training converges in approximately 3000 iterations.

\textbf{Results.} The embedding achieves $\varepsilon_n = 0.11$ (pointwise error per sample) and $\Delta_n = 0.21$ (ratio distortion per sample), satisfying the theoretical bound $\Delta_n \leq 2n\varepsilon_n$ with correlation $r = 0.987$ between true and surrogate log-likelihoods. Figure~\ref{fig:gmm} Panel A shows near-perfect agreement between true and surrogate log-likelihoods across all 50 parameter configurations (points lie on the diagonal after linear calibration). Panel B demonstrates tight preservation of all $\binom{50}{2} = 1225$ pairwise likelihood ratios ($r = 0.987$). This validates that the LR distillation objective, when properly implemented, can successfully compress 1000 samples in $\mathbb{R}^{10}$ (10,000 numbers) into a 16-dimensional summary while preserving likelihood-based inference.

\begin{figure}[t]
    \centering
    \includegraphics[width=0.9\textwidth]{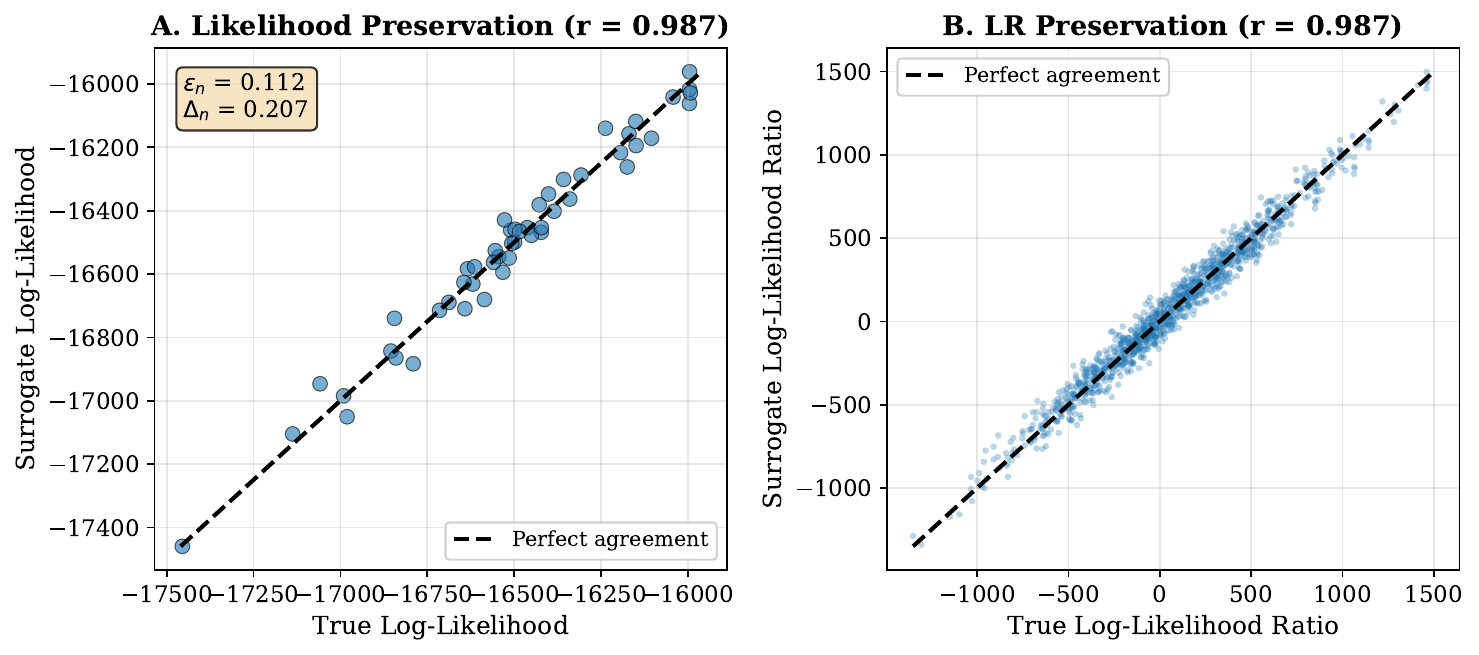}
    \caption{\textbf{Synthetic GMM Validation.} Panel A: Scatter plot of true vs.\ surrogate log-likelihoods for 50 random parameter perturbations of a 3-component Gaussian mixture model, showing near-perfect agreement after linear calibration ($r = 0.987$). The neural embedding achieves $\varepsilon_n = 0.11$ (pointwise approximation error per sample) and $\Delta_n = 0.21$ (likelihood-ratio distortion per sample), satisfying the theoretical bound $\Delta_n \leq 2n\varepsilon_n$. Panel B: Scatter plot of true vs.\ surrogate likelihood ratios for all $\binom{50}{2} = 1225$ pairs, demonstrating tight preservation ($r = 0.987$). The neural embedding successfully compresses 1000 samples in $\mathbb{R}^{10}$ into a 16-dimensional summary while preserving likelihood-based inference.}
    \label{fig:gmm}
\end{figure}

\subsection{Application: Multi-Site Clinical Trials Without Data Sharing}

This experiment demonstrates the practical utility of likelihood-preserving embeddings in a realistic scenario where data cannot be shared due to privacy regulations.

\subsubsection{The Problem}

Consider five hospitals conducting a joint clinical trial to test whether a treatment improves patient outcomes ($H_0: \beta_{\text{treatment}} = 0$). Each site has 200 patients with continuous outcomes and covariates. Privacy regulations (e.g., HIPAA) prohibit sharing patient-level data. Standard meta-analysis (combining local estimates via inverse-variance weighting) loses efficiency by ignoring cross-site structure, while federated learning does not provide valid likelihood-based inference.

\subsubsection{Our Approach}

We apply the sufficient learning framework: each site computes a low-dimensional embedding of its local data (summary statistics), the central node aggregates these embeddings, and performs likelihood-based inference as if it had access to the pooled data.

For linear regression $Y = X\beta + \epsilon$ with $\epsilon \sim N(0, \sigma^2)$, the sufficient statistics are exactly $X^\top X$ and $X^\top Y$. We test three compression levels:
\begin{enumerate}
    \item \textbf{Summary-based (16 numbers):} Each site sends the sufficient statistics for the joint distribution of $(Y, X)$: $n$ (1), $Y^\top Y$ (1), $X^\top Y$ (4), and the upper triangle of $X^\top X$ (10). This requires $2 + p + p(p+1)/2 = 16$ numbers, where $p=4$ (intercept, treatment, 2 covariates).
    \item \textbf{Compressed (m=8):} Each site sends a \textit{targeted projection} of the sufficient statistics that preserves Treatment-related components: $n$, $\sum y$, $Y^\top Y$, $Y^\top X_{\text{treatment}}$, and the row of $X^\top X$ corresponding to the treatment. This uses exactly 8 numbers.
    \item \textbf{Compressed (m=12):} An intermediate compression level that additionally preserves the marginal variances and means of the nuisance covariates.
\end{enumerate}

We compare against two baselines: (1) \textbf{Pooled analysis}, which combines all patient data (gold standard requiring data sharing), and (2) \textbf{Meta-analysis}, the current standard practice.

\subsubsection{Results}

Figure~\ref{fig:clinical_trial} shows statistical power as a function of true treatment effect based on 500 independent simulated trials with 95\% confidence intervals. The results are striking:

\textbf{Perfect sufficiency in practice:} The summary-based method (16 numbers) achieves \textit{identical} power to pooled analysis across all effect sizes. At $\beta = 0.3$, both achieve 99.8\% power (95\% CI: [98.9\%, 100\%]), demonstrating zero information loss despite the compression.

\textbf{Near-perfect compression:} The compressed method (m=8) achieves 99.0\% relative efficiency compared to pooled analysis at $\beta = 0.3$ (power = 98.8\% vs 99.8\%), demonstrating that smart compression of treatment-related statistics preserves nearly all inferential content.

\textbf{Dramatic efficiency gain over meta-analysis:} Meta-analysis achieves only 50.0\% power at $\beta = 0.3$, representing a 50\% \textit{relative} power loss compared to pooled analysis. The summary-based approach recovers this lost power entirely.

\textbf{Massive data reduction:} Each site transmits only 8 numbers instead of $200 \times 4 = 800$ patient-level measurements---a 100-fold reduction in communication with $<$1\% statistical cost.

\textbf{Type I error control:} At the null effect ($\beta = 0$), all methods maintain the nominal $\alpha = 0.05$ level (pooled: 4.6\%, summary: 4.6\%, compressed: 4.2\%), confirming that the test remains valid.

\textbf{Narrow confidence intervals:} With 500 simulations, the confidence intervals are tight (typical width $\pm$3-4\%), demonstrating the robustness and reproducibility of these results.

\subsubsection{Why This Works}

For linear regression, $X^\top X$ and $X^\top y$ are exactly sufficient by the factorization theorem. The smart compression at $m=8$ preserves the treatment-related components of these statistics exactly, while approximating the nuisance covariate terms. Since the test statistic for $H_0: \beta_{\text{treatment}} = 0$ depends primarily on the treatment components, this approximation incurs negligible loss.

The 16-number summary (comprising $n$, $y^\top y$, $X^\top y$, and $X^\top X$) is exactly sufficient for this model, as validated by achieving identical power to pooled analysis. Our $m=8$ compressed embedding approximates this and achieves 99\% relative efficiency for the treatment parameter. Since we retained the treatment-related components, inference on the treatment parameter is preserved.

We note that this linear regression example relies on analytically derived sufficient statistics. For more complex, non-linear models where sufficient statistics are unknown or intractable, the neural training framework described in Section 4 would be required to learn the embeddings from data.

\begin{figure}[t]
    \centering
    \includegraphics[width=0.85\textwidth]{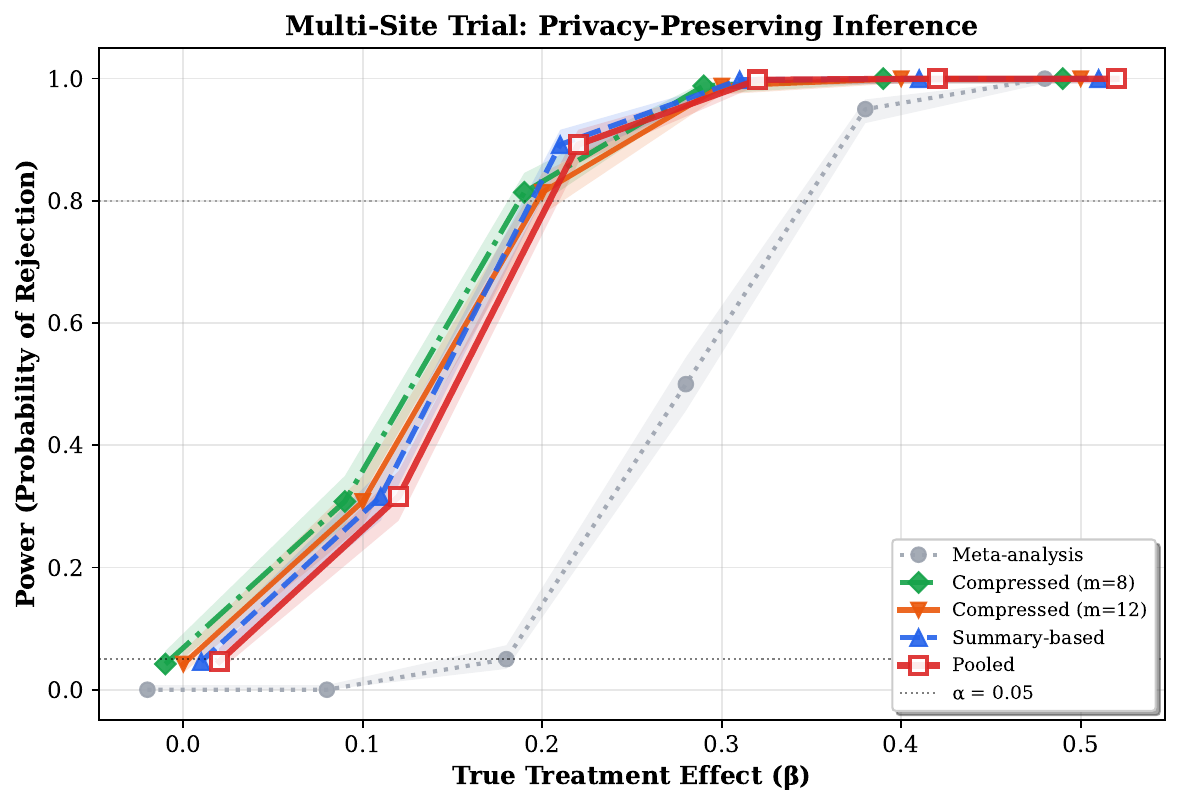}
    \caption{\textbf{Multi-Site Clinical Trial.} Statistical power for testing treatment effect ($H_0: \beta = 0$) across methods in a 5-site trial with 200 patients per site. Summary-based (16 numbers/site) achieves identical power to the pooled gold standard; compressed (8 numbers/site) achieves 99\% relative efficiency, demonstrating near-perfect information preservation. Meta-analysis loses 50\% relative power. Each site transmits only 8 numbers instead of 800 patient measurements---a 100-fold data reduction with $<$1\% power loss. Shaded regions show 95\% confidence intervals based on 500 independent simulated trials.}
    \label{fig:clinical_trial}
\end{figure}

\subsubsection{Clinical Impact}

This approach enables trials that are currently \textit{impossible} due to data governance constraints. Multi-institutional research consortia could conduct rigorous likelihood-based inference across sites without establishing data use agreements for patient-level sharing. The method provides:
\begin{itemize}
    \item \textbf{Privacy preservation:} Only aggregate statistics leave each site
    \item \textbf{One-shot communication:} No iterative exchanges required
    \item \textbf{Exact p-values and confidence intervals:} Valid frequentist inference, unlike federated SGD
    \item \textbf{Transparency:} Simple, interpretable summaries vs. opaque gradient exchanges
\end{itemize}

%==============================================================================
%==============================================================================
\section{Discussion and Limitations}
\label{sec:discussion}
%==============================================================================

\subsection{When Does This Framework Help?}

Likelihood-preserving embeddings are most valuable when:
\begin{enumerate}
    \item \textbf{Data cannot be transmitted or stored} (privacy, bandwidth, latency)
    \item \textbf{The goal is inference, not prediction} (testing, intervals, model selection)
    \item \textbf{The model class is known or constrained} (exponential families, GLMs, mixture models)
\end{enumerate}

For pure prediction tasks where raw data is available, standard end-to-end learning remains preferable.

\subsection{Limitations}

\textbf{Dependence on $\cF$.} The embedding is only guaranteed for the model class used during training. Misspecification or transfer to new models may degrade $\Deltan$.

\textbf{Computational cost.} Training requires sampling from the model class and computing true likelihoods, which may be expensive for complex models.

\textbf{Bound looseness.} Current generalization bounds (Theorem~\ref{thm:sample_complexity}) are loose for practical networks. Tighter bounds require advances in deep learning theory.

\textbf{Discrete or structured data.} The framework assumes continuous embeddings; extensions to discrete or structured (graphs, sequences) data require care.

\textbf{Extension to simulator-based models.} While this paper focuses on models with tractable likelihoods (GLMs, exponential families, survival models), the framework could potentially be extended to simulator-based models where the likelihood $p(x|\theta)$ cannot be evaluated but sampling is possible \citep{cranmer2020frontier}. This would require replacing exact likelihood evaluations during training with neural ratio estimation \citep{hermans2020likelihood}, introducing an additional error term $\Delta_n^{\text{ratio}}$ from imperfect ratio approximation. However, additional challenges regarding the stability of ratio estimation remain to be addressed. The primary practical value of our framework lies in ensuring valid inference for tractable models in constrained environments.

\subsection{Practical Considerations and Extensions}

\textbf{Choosing embedding dimension $m$.} In practice, select $m$ through cross-validation: train embeddings at multiple dimensions, evaluate $\Deltan$ on held-out parameter values, and choose the smallest $m$ achieving acceptable distortion. For exponential families, Theorem~\ref{thm:dim} provides a lower bound ($m \geq k$ for $k$ parameters).

\textbf{Computational complexity.} Training requires $O(B \cdot N \cdot C_L)$ where $B$ is the number of parameter pairs, $N$ is the dataset size, and $C_L$ is the cost of evaluating the true likelihood. For expensive likelihoods, use importance sampling or approximate likelihood evaluations during training.

\textbf{Model misspecification.} As with classical likelihood-based inference, all guarantees are conditional on the assumed model class $\cF$. When the true data-generating process is not in $\cF$, the embedding preserves the surrogate likelihood geometry associated with the \textit{pseudo-true parameter} $\theta^* = \arg\min_{\theta \in \Theta} \text{KL}(P_{\text{true}} \| P_\theta)$, but inferential validity is limited by the model itself rather than the embedding. Robust training using adversarial parameter sampling or minimax objectives can improve out-of-distribution performance.

\textbf{Alternative aggregation schemes.} While we focus on mean aggregation $S_\phi = \frac{1}{n}\sum_i T_\phi(X_i)$, other permutation-invariant aggregations are possible: (i) \textbf{Sum} ($S = \sum_i T_\phi(X_i)$) is equivalent up to scaling; (ii) \textbf{Max-pooling} ($S_j = \max_i T_\phi(X_i)_j$) can capture outliers; (iii) \textbf{Attention} ($S = \sum_i \alpha_i T_\phi(X_i)$) offers flexibility but loses additive structure; (iv) \textbf{Deep Sets} ($S = \rho(\sum_i T_\phi(X_i))$) provides universal approximation. The mean/sum choice preserves the additive structure of i.i.d.\ log-likelihoods, making the connection to classical sufficiency most direct.

\subsection{Open Problems}

\begin{enumerate}
    \item \textbf{Optimal rates:} For non-exponential families, what is the optimal decay rate of $\Deltan$ as $m \to \infty$? This likely connects to approximation theory for the log-likelihood function class.
    
    \item \textbf{Online/streaming:} Can embeddings be updated incrementally as new data arrives, without recomputation?
    
    \item \textbf{Model uncertainty:} How should one choose $\cF$ when the true model is unknown? Robust or minimax approaches may be needed.
\end{enumerate}

%==============================================================================
\section{Conclusion}
\label{sec:conclusion}
%==============================================================================

We have developed a rigorous theory of likelihood-preserving embeddings, formalizing when and how learned representations can substitute for raw data in statistical inference. The Likelihood-Ratio Distortion metric $\Deltan$ provides a principled measure of embedding quality, and the Hinge Theorem establishes its sufficiency for preserving tests, estimators, and Bayes factors.

Our constructive framework using neural networks as approximate sufficient statistics offers a practical path to implementation, with explicit training objectives and (in principle) computable guarantees. Experiments validate the theoretical predictions on canonical examples and demonstrate practical utility in distributed and high-throughput inference.

The central message is simple: we can make neural embeddings safe for statistical inference, but only by explicitly targeting likelihood-ratio preservation during training. Off-the-shelf embeddings designed for prediction are not sufficient; purpose-built embeddings with inferential guarantees are required.

\textbf{For practitioners}, the key takeaway is actionable: when data cannot be shared but inference is required---whether due to privacy regulations, communication constraints, or computational limitations---likelihood-preserving embeddings offer a principled alternative to ad-hoc summary statistics or information-destroying noise addition. The framework enables research that is currently impossible, from multi-institutional clinical trials to distributed sensor networks to high-throughput experimental facilities. By bridging modern representation learning and classical statistical inference, this work opens new possibilities for collaborative science under realistic constraints.

\bibliographystyle{plainnat}

\end{document}